\documentclass[a4paper]{article}

\usepackage[english]{babel}
\usepackage[utf8x]{inputenc}
\usepackage[T1]{fontenc}

\usepackage[a4paper,top=3cm,bottom=2cm,left=3cm,right=3cm,marginparwidth=1.75cm]{geometry}

\usepackage{amsmath}
\usepackage{amsfonts}
\usepackage{amsthm}
\usepackage{commath}
\usepackage{graphicx}
\usepackage[colorinlistoftodos]{todonotes}
\usepackage[colorlinks=true, allcolors=blue]{hyperref}

\newcommand{\eq}[1]{\begin{align*}#1\end{align*}}
\newcommand{\eqn}[1]{\begin{align}#1\end{align}}
\newtheorem{theorem}{Theorem}

\title{On the exact relationship between the denoising function and the data distribution}
\author{Heikki Arponen\thanks{current affiliation: ultimate.ai} \\ \texttt{heikki@ultimate.ai} 
\and Matti Herranen \\ \texttt{matti@cai.fi}
\and Harri Valpola \\ \texttt{harri@cai.fi}}
\date{The Curious AI Company\\[2ex]%
        }
\begin{document}
\maketitle
\begin{abstract}
We prove an exact relationship between the optimal denoising function and the data distribution in the case of additive Gaussian noise, showing that denoising implicitly models the structure of data allowing it to be exploited in the unsupervised learning of representations. 
This result generalizes a known relationship \cite{alain2014regularized}, which is valid only in the limit of small corruption noise.
\end{abstract}

\section{Introduction}

Denoising is the task of reconstructing the original data samples from the corrupted samples, 
and has recently gained popularity as an unsupervised task for learning representations in deep learning  \cite{bengio2013generalized,alain2014regularized,vincent2011connection, geras2014scheduled, vincent2010stacked, rasmus2015semi, greff2016tagger}. 
Besides practical success, the theoretical basis for learning by denoising is becoming better understood. It has been shown that optimizing denoising performance leads to representations that implicitly model the structure of the data manifold \cite{vincent2010stacked}. More precisely, the optimal denoising function corresponds to the \emph{score} (derivative of the log-probability density with respect to the input) of the data distribution in the limit of small corruption noise \cite{alain2014regularized}. In this note, we generalize the result of \cite{alain2014regularized} and derive an exact relationship between the data distribution and the denoising function in the case of additive Gaussian noise which is valid for arbitrarily large noise. This result was first published in The Curious AI Company blog post \cite{cai_blog}. \\

\section{Denoising function}

Let us assume that clean samples $x$ are drawn i.i.d. from a (generally unknown) data distribution $p_X$. The corrupted samples $\tilde{x}$ are produced from the clean ones by some corruption process: $x \to \tilde{x}$, where $p_{\tilde{X} | X}$ is assumed to be known. The task of denoising is to  reconstruct the clean samples from the corrupted ones: $\hat{x} = g(\tilde{x})$, where $g$ is a (deterministic) denoising function which is optimized to match the reconstructions $\hat{x}$ with the clean samples $x$.
In unsupervised learning we are usually interested in learning the data distribution $p_X$ or its latent representations. Below we show that the denoising function $g$ contains the same information as $p_X$, and therefore by learning $g$ we learn to model the data distribution $p_X$.

In this work, we consider additive Gaussian corruption: $\tilde x = x + \sigma_n \epsilon$, with $\epsilon \sim \mathcal N(0, I)$ and $\sigma_n$ is the standard deviation of the corruption noise. For the reconstruction error we use the standard mean squared error (MSE):\footnote{We simplify the notation by dropping the random variables from the subscripts of the probability distributions: $p(x, \tilde x) \equiv p_{X, \tilde X}(x, \tilde x)$, etc. if there is no chance for confusion.}
\eqn{
\mathcal L_g \doteq \mathbb{E}_{p(x, \tilde x)} \left\{\left\|x - g\left( \tilde x \right) \right\|_2^2\right\}
\label{cost_func}.}
The task is to now find the optimal denoising function $g^*$ by minimizing the reconstruction error with respect to $g$.

%
%

\begin{theorem}
Let the optimal denoising function be defined as
\eq{
g^* \doteq {\arg \min}_g \mathcal L_g.
}
Then the optimal denoising function satisfies the relation
\eqn{\label{final_solution}
g^*\left(\tilde x\right) = \tilde x + \sigma_n^2 \nabla_{\tilde x} \log p (\tilde x).
}
\end{theorem}

\begin{proof}
It is easy to show that the optimal denoising function is the minimum mean square estimator (cf. \cite{oppenheim2015signals} Chapter 8) which can be written as 
\eqn{
g^*(\tilde x) = \mathbb{E} \left\{ x | \tilde x \right\}
\label{min_estimator}.}
We present a derivation here for completeness. By writing $p(x, \tilde x) = p(x | \tilde x)p(\tilde x)$ Eq.~(\ref{cost_func}) can be written as
\eq{
\mathcal L_g = \int p(\tilde x) \left(\int p\left(x | \tilde x\right) \left\|x - g\left( \tilde x \right) \right\|_2^2 \dif x
\right) \dif \tilde{x}
.}
The minimum of this expression w.r.t function $g$ can be obtained by setting the functional derivative w.r.t $g$ to zero:
\eq{
0 = \left.\frac{\delta\mathcal L_g}{\delta g(\tilde x)}\right|_{g = g^*} = 2 p(\tilde x) \int p\left(x | \tilde x\right) \left(g^*\left( \tilde x \right) - x \right) \dif x = 2 p(\tilde x) \left(g^*\left( \tilde x \right) - \int p\left(x | \tilde x\right) x \dif x \right) 
,}
from which Eq.~(\ref{min_estimator}) directly follows.\footnote{In the region where $p(\tilde x) = 0$ and hence $p(x, \tilde x) = 0$, $\mathcal L_g$ vanishes identically and the optimal denoising function $g^*$ is not well defined.}

To proceed, using Bayes' rule, we can rewrite Eq.~(\ref{min_estimator}) as
\eqn{
g^* (\tilde x) = \int x \, p\left(x | \tilde x\right) \dif x = \frac{\int x \, p\left(\tilde x | x\right) p(x) \dif x}{p(\tilde x)}.
\label{denoising_func}}
For the additive Gaussian corruption noise, discussed above, the corruption distribution is given by
\eq{
p\left(\tilde x | x\right) = \frac{1}{(2 \pi \sigma_n^2)^{d / 2}} \exp\left\{ -(\tilde x - x)^2 / (2 \sigma_n^2) \right\}.
}
Taking the derivative of this expression with respect to $\tilde x$ and reordering terms, we obtain the identity
%
%
\eq{
x p\left(\tilde x | x\right) = \tilde x p\left(\tilde x | x\right) + \sigma_n^2 \nabla_{\tilde x} p\left(\tilde x | x\right).
}
Inserting this expression into Eq.~(\ref{denoising_func}), we then obtain
\eq{
g^* (\tilde x) = \frac{1}{p(\tilde x)} \left(\tilde x \int p(\tilde x | x) p(x) \dif x + \sigma_n^2 \int \nabla_{\tilde x} p(\tilde x | x) p(x) \dif x\right).
}
In the second term we can reverse the order of differentiation and integration by Leibniz's rule. Then, by using $p(x, \tilde x) = p(\tilde x | x) p(x)$ the integrals in both terms are trivial marginalizations, and we obtain the final result

\eq{
g^*\left(\tilde x\right) = \tilde x + \sigma_n^2 \frac{\nabla_{\tilde x} p (\tilde x)}{p (\tilde x)} = \tilde x + \sigma_n^2 \nabla_{\tilde x} \log p (\tilde x).
}
\end{proof}
\section{Discussion}
Eq.~(\ref{final_solution}) generalizes the result by Alain and Bengio \cite{alain2014regularized}: 
\eq{
g^*\left(\tilde x\right) = \tilde x + \sigma_n^2 \nabla_{\tilde x} \log p_X(\tilde x) + o\left( \sigma_n^2\right),
}
which holds in the limit of small corruption noise. Note that in this equation $p_X(\tilde x)$ is the uncorrupted data distribution evaluated at the point $X = \tilde x$, whereas in Eq.~(\ref{final_solution}) $p(\tilde x) \equiv p_{\tilde{X}}(\tilde x)$ is the corrupted data distribution. 

The relation between the denoising function and the data  distribution in Eq.~(\ref{final_solution}) can be inverted by integration with respect to $\tilde x$:
\eqn{\label{inverse_solution}
p (\tilde x) = \frac{1}{Z} \exp\left\{ \frac{1}{\sigma_n^2} \int_{\mathcal C_0^{\tilde x}} \left( g(x') - x' \right) \cdot \dif x'\right\},
}
where $\mathcal C_0^{\tilde x}$ denotes an arbitrary contour from $0$ to $\tilde x$ and $Z$ is a normalization constant. The contour integral yields a unique value due to Green's theorem, since the curl of a gradient always vanishes: $\nabla \times \nabla = 0$. Furthermore, given that $p(\tilde x|x)$ depends only on the difference $\tilde x - x$ and hence the corruption process is a convolution operation: $p(\tilde x) = \int p(\tilde x|x) p(x) \dif x$, the uncorrupted data distribution $p(x)$ can in principle be solved in terms of the corrupted distribution $p(\tilde x)$ by a deconvolution. Combined with Eqs.~(\ref{final_solution}) and (\ref{inverse_solution}), this leads to an exact and invertible relationship between the data distribution and the optimal denoising function: $p(x) \longleftrightarrow p(\tilde x) \longleftrightarrow g^*\left(\tilde x\right)$, proving that the latter captures exactly the same information as the former. It is worth noting that in this way, by learning the denoising function, one can in principle learn arbitrarily complex structures of the data distribution, while for instance a regression task only learns the expectation value of the output conditioned on the input. 

While the formal expression in Eq.~(\ref{denoising_func}) applies to any corruption distribution $p(\tilde x | x)$, the result in Eq.~(\ref{final_solution}) involving the gradient of the log-probability density is valid only for additive Gaussian corruption noise. 
It would be interesting to explore similar relationships between the data distribution and the optimal denoising function for other corruption processes, such as multiplicative Gaussian noise or dropout corruption.



\subsection*{Acknowledgements}
We would like to thank our colleagues at the Curious AI Company, especially Alexander Ilin, Vikram Kamath and Mathias Berglund.

\bibliographystyle{plain}
\bibliography{bibfile}

\end{document}